\documentclass{article}
\PassOptionsToPackage{numbers,sort&compress}{natbib}
\usepackage[dblblindworkshop, final]{neurips_2026}
\workshoptitle{The 6th Workshop on Mathematical Reasoning and AI (MATH-AI)}

\usepackage[utf8]{inputenc}
\usepackage[T1]{fontenc}
\usepackage{amsmath,amsfonts,amssymb,amsthm}
\usepackage{booktabs}
\usepackage{graphicx}
\usepackage{microtype}
\usepackage{xcolor}
\usepackage{xspace}
\usepackage{hyperref}
\hypersetup{colorlinks=true,linkcolor=black,citecolor=black,urlcolor=black}
\usepackage{url}
\usepackage{needspace}
\usepackage[most]{tcolorbox}

\newtheoremstyle{appendixplain}
  {12pt plus 2pt minus 1pt}{10pt plus 2pt minus 1pt}
  {\itshape\setlength{\parskip}{4pt}}{}
  {\bfseries}{.}{\newline}
  {\thmname{#1}\thmnumber{ #2}\thmnote{\normalfont\ (#3)}}
\newtheoremstyle{appendixdefinition}
  {12pt plus 2pt minus 1pt}{10pt plus 2pt minus 1pt}
  {\normalfont\setlength{\parskip}{4pt}}{}
  {\bfseries}{.}{\newline}
  {\thmname{#1}\thmnumber{ #2}\thmnote{\normalfont\ (#3)}}
\theoremstyle{appendixplain}
\newtheorem{theorem}{Theorem}[section]
\newtheorem{proposition}[theorem]{Proposition}
\newtheorem{lemma}[theorem]{Lemma}
\newtheorem{corollary}[theorem]{Corollary}
\theoremstyle{appendixdefinition}
\newtheorem{definition}[theorem]{Definition}
\newtheorem{assumption}[theorem]{Assumption}
\AddToHook{env/theorem/before}{\Needspace{6\baselineskip}}
\AddToHook{env/proposition/before}{\Needspace{6\baselineskip}}
\AddToHook{env/lemma/before}{\Needspace{6\baselineskip}}
\AddToHook{env/corollary/before}{\Needspace{6\baselineskip}}
\AddToHook{env/definition/before}{\Needspace{6\baselineskip}}
\AddToHook{env/assumption/before}{\Needspace{13\baselineskip}}
\AddToHook{env/proof/before}{\Needspace{5\baselineskip}}
\makeatletter
\renewenvironment{proof}[1][\proofname]{%
  \par\pushQED{\qed}\normalfont
  \topsep6pt\relax\trivlist
  \item[\hskip\labelsep\itshape #1\@addpunct{.}]\mbox{}\par\nobreak
  \vspace{3pt}\noindent\ignorespaces
}{%
  \popQED\endtrivlist\@endpefalse
}
\makeatother
\newcommand{\GeoLib}{\textsc{GeoLib}\xspace}
\newcommand{\GeoDL}{\textsc{GeoDL}\xspace}
\newcommand{\OmniGeo}{\textsc{Omni-Geo}\xspace}
\newcommand{\gpred}[1]{\operatorname{#1}}

\newcommand{\GoalCheck}{\operatorname{CheckGoal}}

\newcommand{\ind}{\mathbb{I}}

\newtcblisting{promptbox}[1]{listing only,enhanced,breakable,colback=white,
  colframe=black,boxrule=.55pt,arc=0pt,title={#1},title after break={#1
  (continued)},colbacktitle=white,coltitle=black,fonttitle=\bfseries\footnotesize,
  left=4pt,right=4pt,top=2pt,bottom=2pt,before skip=4pt,after skip=4pt,
  listing options={basicstyle=\ttfamily\fontsize{7.5pt}{8.2pt}\selectfont,
  breaklines=true,columns=fullflexible,keepspaces=true,showstringspaces=false}}

\title{Dense Neuro-Symbolic Reasoning\\in a Unified Geometry State}

\author{
Ruoran Xu \quad Wending Gao \quad Haoyu Cheng \quad Xiaoqiang Kang \quad
Qiufeng Wang\thanks{Corresponding author: \texttt{qiufeng.wang@xjtlu.edu.cn}}\\
Xi'an Jiaotong-Liverpool University
}

\begin{document}
\maketitle
\setcounter{footnote}{0}

\begin{abstract}
Geometry reasoning is naturally stateful: solving a problem repeatedly alternates
between structural proposals and exact deductions. We formulate this process as
\emph{dense neural--symbolic coupling}, in which neural guidance and symbolic
execution share a typed state and communicate through executable actions at every
search step. Neural proposals contribute theorem instances, constructions, and
algebraic bridges; the symbolic runtime applies registered rules, propagates exact
constraints, and records provenance. A nested controller allocates computation first
between neural and symbolic proposal sources and then among admitted actions. We
instantiate the framework in \OmniGeo, a single solver for plane, analytic (including
conics), and solid geometry. With Claude Sonnet 4.6, \OmniGeo reaches 94.2\%,
88.5\%, and 89.8\% on FormalGeo7K, Conic10K, and SolidFGeo, respectively
(90.8\% macro average), and solves 21/30 IMO-AG-30 problems.
\end{abstract}

\section{Geometry reasoning as dense neural--symbolic interaction}
\label{sec:introduction}

Geometry reasoning unfolds through intermediate mathematical states. The statement
and diagram establish objects, relations, and quantities; subsequent steps introduce
auxiliary structure, derive exact consequences, and connect them to the goal. Neural
models are well suited to proposing structural moves from the full problem context,
while symbolic engines execute registered theorems and algebraic consequences under
explicit semantics. Their natural meeting point is the evolving formal state.

We call their repeated state-level interaction \emph{dense neural--symbolic
coupling}. Every admitted action is executed before the next decision, and both
reasoning sources condition on the updated state:
$s_t\!\rightarrow\!\mathcal C(s_t)\!\rightarrow\!a_t\!\rightarrow\!
s_{t+1}=T(s_t,a_t)$, where $\mathcal C(s_t)$ is the admitted candidate set. A neural
construction can therefore expose new symbolic consequences, while a symbolic
deduction can immediately reshape the next neural proposal. The complete reasoning
trace is built through this continuous feedback loop over the shared formal state.

This formulation leads to three coupled requirements: a \emph{shared representation}
covering heterogeneous geometry domains, a \emph{typed transition contract} that
records the logical status and provenance of every accepted action, and
\emph{adaptive allocation} over both proposal sources and admitted continuations.
\OmniGeo realizes these requirements with \GeoLib and \GeoDL, a common action interface,
and nested source-UCB and action-PUCT search. Prior geometry systems establish structured
text--diagram reasoning, learned theorem guidance, executable formal states, and
neural-guided search~\citep{seo2015geos,lu2021intergps,zhang2025fgeohypergnet,
xu2026hilbert,lample2022htps,xin2025deepseekprover,zhou2024lats}; \OmniGeo organizes
these capabilities around one continually updated typed geometry state.

\paragraph{Contributions.}
\begin{enumerate}
\item \textbf{Dense neural--symbolic coupling.} A shared typed state, common
execution interface, explicit provenance, and nested source/action allocation make
neural and symbolic progress immediately available to subsequent reasoning.
\item \textbf{Unified geometry reasoning with \OmniGeo.} One formal language and
solver spans plane, analytic (including conics), and solid geometry; benchmark and
ablation studies measure the representation, proposal, admission, and allocation
components across three main benchmarks and IMO-AG-30.
\end{enumerate}

\begin{figure}[t]
  \centering
  \includegraphics[width=\linewidth]{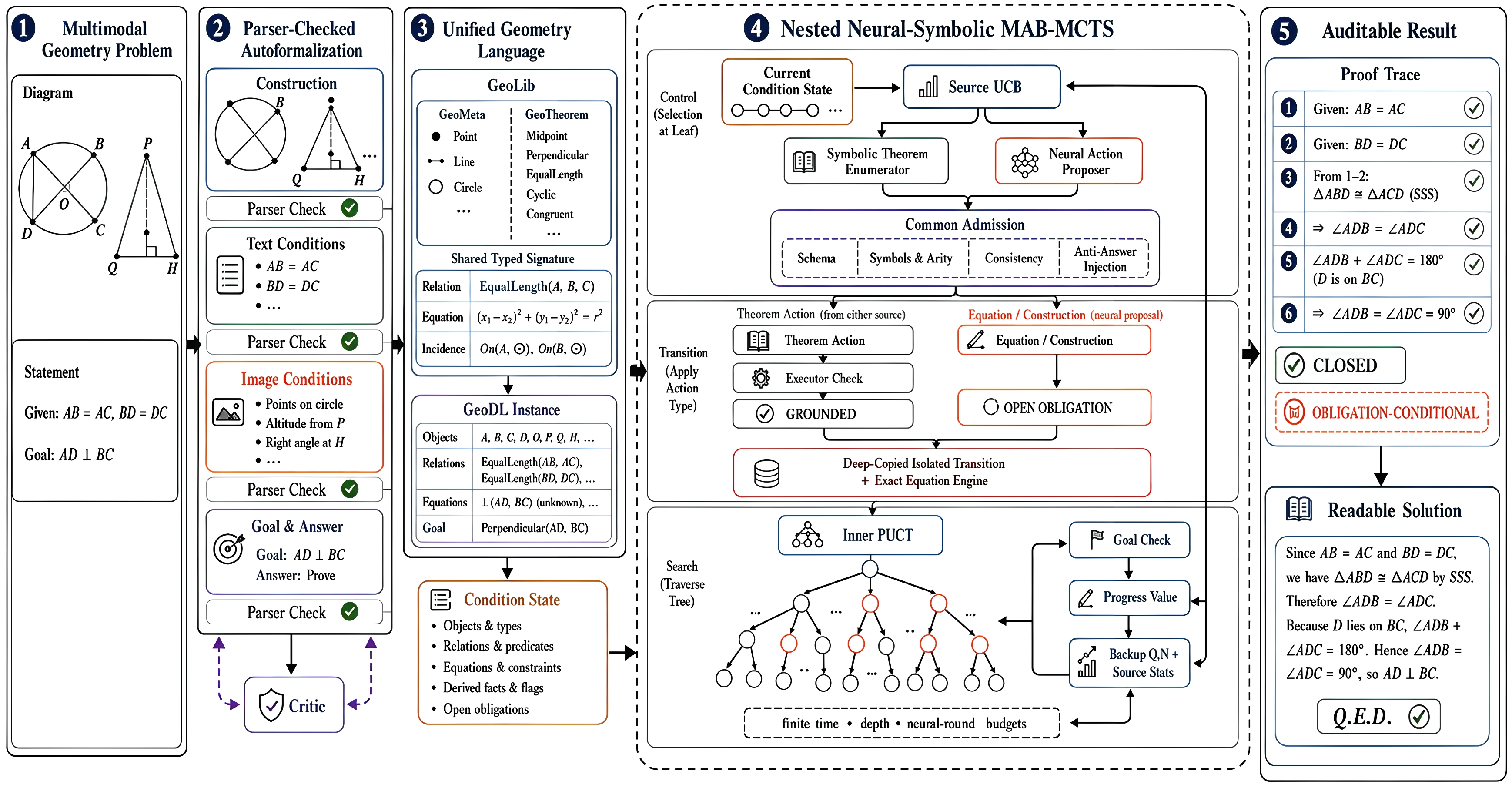}
  \caption{\OmniGeo overview. Staged formalization initializes a shared typed
  state, which neural guidance and symbolic rollout extend through adaptive
  search. Execution feedback informs subsequent steps, while provenance
  records their logical dependencies.}
  \label{fig:overview}
\end{figure}

\section{Unified geometry state and typed transitions}
\label{sec:state}

\paragraph{Shared formal state.}
\OmniGeo instantiates the coupling through a common representation of
geometric objects and relations. \GeoLib declares their types, predicates,
attributions, and executable theorem schemas. \GeoDL represents each problem
as $(\Gamma_0,G)$:
typed input conditions and a typed goal. All three geometry families use one
parser, condition store, theorem matcher, equation engine, and goal checker.
The compiler first extracts constructions, then conditions from the text and
diagram, and finally the goal. Parser feedback and critic-guided repair refine
this formalization (Fig.~\ref{fig:overview}), establishing the state on which
all subsequent reasoning operates.

\paragraph{Typed transition contract.}
To maintain a consistent interface throughout search, we define the state as
$s=(\Gamma_s,\mathcal E_s,G,\mathcal O_s,\Pi_s)$: conditions, exact equations,
the goal, open obligations, and a provenance ledger. Each candidate is a typed
action $a=(\tau,\eta,P,\sigma)$ containing its action class, payload, prior,
and source. Candidates from both sources undergo common admission checks for schema,
symbols and arity, duplicates, consistency, and direct answer injection.
Registered theorem actions are rematched and executed by the engine, with
their premise dependencies recorded. Neural constructions or asserted
equations can extend the state while retaining \emph{open} obligations until
grounded by input conditions or subsequent exact derivations. Each admitted
action is executed on an isolated copy of the parent state under the same
transition system:
\begin{equation}
 s' = T(s,a),\qquad
 \mathcal K(s)=\mathcal K_{\rm thm}(s)\cup\mathcal K_{\rm exact}(s)
 \cup\mathcal K_{\rm open,\theta}(s).
 \label{eq:kernel}
\end{equation}
The three action sets correspond to registered theorem application, exact
propagation, and open neural extensions. Their common execution contract lets
deductions and structural proposals compose while preserving their distinct logical status
(Appendix~\ref{app:language}).

\section{Adaptive dual-source search}
\label{sec:search}

Once both reasoning sources act on the same state, source choice becomes a search
decision. At each newly expanded leaf, the outer controller selects
$E\in\{\mathrm{sym},\mathrm{neu}\}$. The symbolic source enumerates registered
theorem instances grounded in $\Gamma_s$; the neural source uses a bounded
planner/extractor call to propose theorem instances, constructions, or equations.
Both sources pass through the same admission layer and materialize typed child states
(Fig.~\ref{fig:nested-search}). If $\mathcal C_E(s)$ is the post-filter child set,
source selection uses
\begin{equation}
 E^*=\arg\max_E\left[\bar R_E(s)+c_{\rm meta}
 \sqrt{\frac{\log(1+N_{\rm src}(s))}{1+n_E(s)}}\right],
 \quad R_E(s)=V(s)\ind[\mathcal C_E(s)\neq\varnothing].
 \label{eq:source}
\end{equation}
Here $n_E(s)$ counts source attempts and $\bar R_E(s)$ is the empirical source
return. A nonempty expansion receives the current progress value $V(s)$; an empty
primary expansion receives zero and gives the other available source an independent
fallback attempt. The controller therefore allocates proposal effort according to
where symbolic enumeration or neural structural proposals are productive along the
evolving path.

\begin{figure}[t]
  \centering
  \includegraphics[width=0.94\linewidth,trim=0 360 0 0,clip]{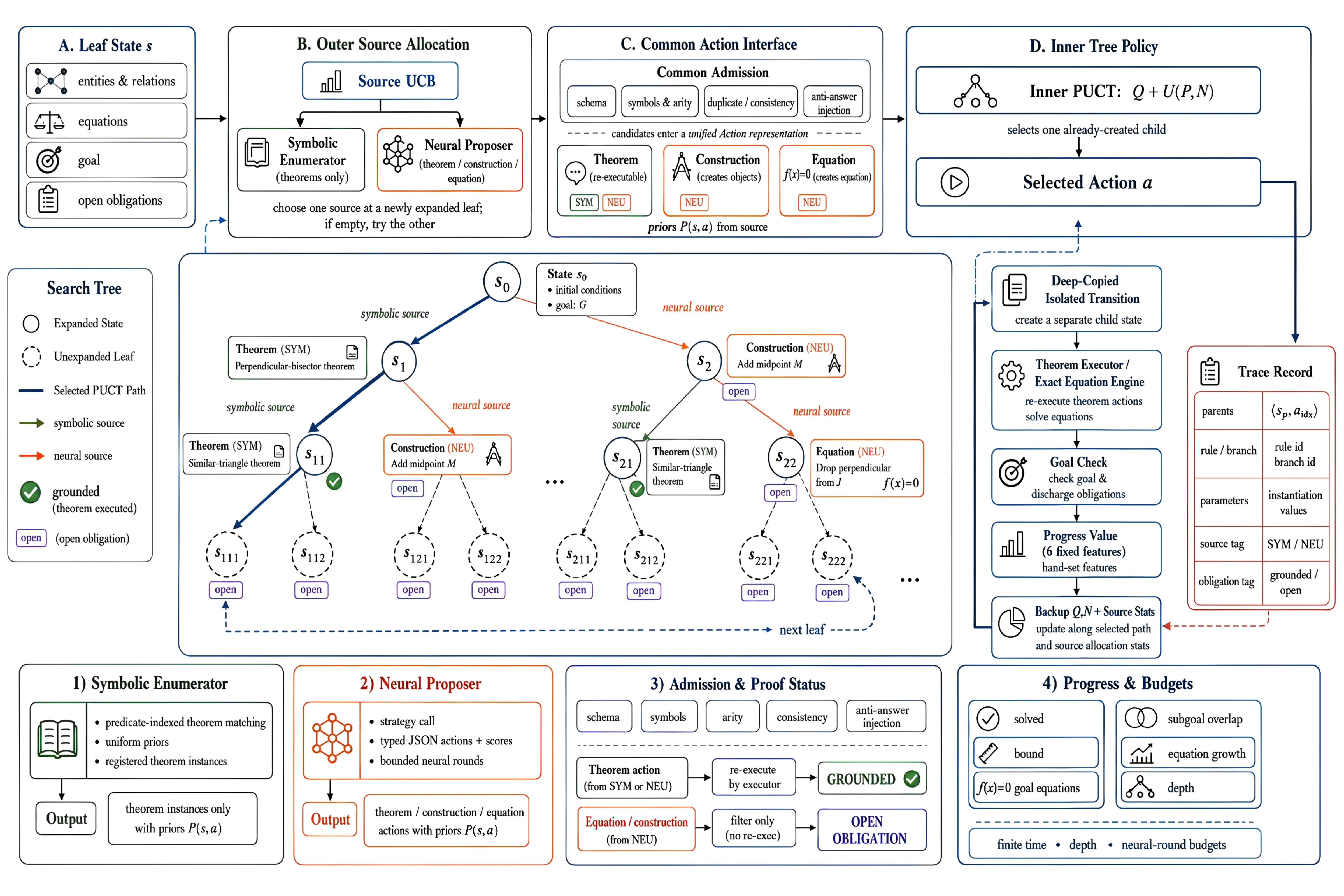}
  \caption{Nested source/action search. Source-UCB chooses symbolic enumeration or
  neural proposal at a leaf; both enter one typed admission interface. Inner PUCT
  traverses admitted children, while isolated execution, goal checking, progress
  evaluation, and provenance feed the next search decision.}
  \label{fig:nested-search}
\end{figure}

After expansion, the inner controller allocates revisits among admitted actions:
\begin{equation}
 a^*=\arg\max_a\left[Q(s,a)+c_{\rm puct}P(s,a)
 \frac{\sqrt{N(s)}}{1+N(s,a)}\right].
 \label{eq:puct}
\end{equation}
$P(s,a)$ is the post-filter prior, $Q(s,a)$ the mean backed-up value, and
$N(s),N(s,a)$ the visit counts. The bounded progress value $V(s)\in[0,1]$
combines solved status, concrete goal bindings, goal-related equations, backward
subgoal overlap, equation growth, and depth; the exact $\GoalCheck$ alone determines
closure. Reported runs use $c_{\rm meta}=c_{\rm puct}=1.5$, maximum depth 15,
three neural expansion rounds, and a nominal 600-second budget per problem.

\paragraph{Formal properties.}
The same transition protocol supports the logical and operational statements used in
search. Provenance-complete derivations are conditionally sound relative to the
compiled inputs and any open obligations in their dependency slice; node-local state
copies preserve branch isolation; finite budgets give termination; and a persistently
live positive-prior child is explored infinitely often whenever its parent is revisited
infinitely often. Appendix~\ref{app:proofs} gives the statements and proofs, including
type preservation, local source-regret, and conditional proof-path reachability.

\section{Experiments}
\label{sec:experiments}

\paragraph{Evaluation protocol.}
We evaluate end-to-end solving on FormalGeo7K (plane; $N=1{,}396$), Conic10K
(analytic/conic; $N=1{,}901$), and SolidFGeo (solid; $N=1{,}320$)
\citep{zhang2023formalgeo,wu2023conic10k,wang2025solidgeo}. FormalGeo7K and
SolidFGeo use official statement--diagram pairs; Conic10K uses its text input. A
prediction is correct only when the engine closes the goal and the external scorer
matches the derived answer to the held-out reference. General-purpose baselines use
the same available modality and answer normalization; DeepSeek-V4-Flash is explicitly
text-only. Within each \OmniGeo configuration, the compiler and in-search proposer
share one frozen backend, while the typed runtime, search policy, and evaluation
protocol remain fixed. Appendix~\ref{app:impl} gives complete prompts and budgets.

\begin{table}[t]
\centering
\caption{End-to-end accuracy (\%) on the three main benchmarks. \emph{Avg.}
is the unweighted macro average over FormalGeo7K, Conic10K, and SolidFGeo.}
\label{tab:main-results}
\small
\setlength{\tabcolsep}{4.1pt}
\begin{tabular}{lrrrr}
\toprule
Method & FormalGeo7K & Conic10K & SolidFGeo & Avg.\\
\midrule
\multicolumn{5}{l}{\emph{General-purpose models}}\\[-2pt]
Claude Sonnet 4.6 & 71.6 & 78.2 & 61.7 & 70.5\\
GPT-5.4 & 85.4 & 82.4 & 65.4 & 77.7\\
Gemini-3-Pro & 85.3 & 79.8 & 63.8 & 76.3\\
DeepSeek-V4-Flash (text only) & 61.2 & 84.6 & 63.0 & 69.6\\
\addlinespace[2pt]
\multicolumn{5}{l}{\emph{Geometry-specific systems}}\\[-2pt]
Hilbert-Geo~\citep{xu2026hilbert} & 84.2 & -- & 78.7 & --\\
\midrule
\textbf{\OmniGeo (Claude Sonnet 4.6)} & \textbf{94.2} & \textbf{88.5} & \textbf{89.8} & \textbf{90.8}\\
\OmniGeo (Qwen3-VL-8B-Thinking) & 92.5 & 85.1 & 82.7 & 86.8\\
\bottomrule
\end{tabular}
\end{table}

\paragraph{Main benchmark results.}
Table~\ref{tab:main-results} shows that \OmniGeo improves end-to-end accuracy across
all three domains. Relative to direct Claude Sonnet 4.6, the Claude-based system gains
$22.6$, $10.3$, and $28.1$ points on FormalGeo7K, Conic10K, and SolidFGeo, raising
the macro average from $70.5\%$ to $90.8\%$ ($+20.3$ points). Relative to the
strongest displayed general-purpose score in each column, the margins are $8.8$,
$3.9$, and $24.4$ points; relative to the strongest complete macro average (GPT-5.4,
$77.7\%$), the difference is $13.1$ points. The same runtime and search stack with
Qwen3-VL-8B-Thinking reaches $92.5\%$, $85.1\%$, and $82.7\%$ ($86.8\%$ macro),
showing that the cross-domain result persists under a second neural backend. Where
Hilbert-Geo reports comparable entries, the Claude-based system is higher by $10.0$
points on FormalGeo7K and $11.1$ points on SolidFGeo.

\begin{table}[t]
\centering
\caption{Component ablations with Claude Sonnet 4.6 as the neural backend.
$\Delta$ Avg. is the absolute percentage-point change from the full system.}
\label{tab:component-ablation}
\small
\setlength{\tabcolsep}{3.4pt}
\begin{tabular}{lrrrrr}
\toprule
Variant & FormalGeo7K & Conic10K & SolidFGeo & Avg. & $\Delta$ Avg.\\
\midrule
Full \OmniGeo & 94.2 & 88.5 & 89.8 & 90.8 & --\\
Fixed source schedule & 89.1 & 83.5 & 84.5 & 85.7 & $-5.1$\\
No neural bridge actions & 80.9 & 65.0 & 68.6 & 71.5 & $-19.3$\\
No admission filter & 87.8 & 79.5 & 82.6 & 83.3 & $-7.5$\\
One-pass formalization & 84.3 & 79.0 & 80.5 & 81.3 & $-9.6$\\
\bottomrule
\end{tabular}
\end{table}

\paragraph{Component analysis.}
Table~\ref{tab:component-ablation} changes one component at a time. A fixed source
schedule reduces the macro average by $5.1$ points; removing neural bridge actions by
$19.3$; removing admission filtering by $7.5$; and one-pass formalization by $9.6$.
Every ablation reduces accuracy on every benchmark. Neural bridge removal produces
the largest drop in all three domains ($13.3$, $23.5$, and $21.2$ points), quantifying
the value of construction and algebraic transitions beyond registered theorem
enumeration. The fixed-schedule drop remains tightly grouped at $5.0$--$5.3$ points,
while the state-side formalization and admission ablations produce consistent losses,
separating the contributions of state construction, transition quality, and adaptive
search control.

\begin{table}[t]
\centering
\caption{Geometry-specific specialization regimes on IMO-AG-30. Resource
quantities are reported in the native units of the corresponding systems.}
\label{tab:imo-ag30}
\small
\setlength{\tabcolsep}{4.0pt}
\begin{tabular}{lclc}
\toprule
Method & Parameter adaptation & Specialized resource & Solved\\
\midrule
AlphaGeometry~\citep{trinh2024alphageometry} & from-scratch LM & 100M theorem--proof examples (9M aux.) & 25/30\\
TongGeometry~\citep{zhang2026tong} & policy/value fine-tuning & 6.7B generated aux. problems & 30/30\\
\OmniGeo & None & $<200$ registered executable schemas & 21/30\\
\bottomrule
\end{tabular}
\end{table}

\paragraph{Olympiad geometry and specialization regime.}
On IMO-AG-30, the Claude-based \OmniGeo configuration solves $21/30$ problems
($70.0\%$) with a frozen neural backend and fewer than 200 registered executable
schemas (Table~\ref{tab:imo-ag30}). AlphaGeometry reports $25/30$ after training from
scratch on 100M theorem--proof examples (9M with auxiliary constructions), while
TongGeometry reports $30/30$ after large-scale generation of 6.7B auxiliary-construction
problems and policy/value fine-tuning~\citep{trinh2024alphageometry,zhang2026tong}.
The table keeps each system's specialization resource in its native unit and places
\OmniGeo's compact executable theory in the same empirical context.

\section{Analysis}
\label{sec:analysis}

\paragraph{State-level composition across geometry families.}
Equation~\ref{eq:kernel} gives a single transition view of three kinds of progress:
registered theorem execution, exact algebraic propagation, and neural bridge actions.
They operate over the same typed condition state, so a construction proposed from
problem context can immediately create premises for theorem matching, and a symbolic
consequence can immediately change the context presented to the neural proposer. The
main benchmark profile shows this composition under one parser, condition store,
equation engine, and search implementation across plane, conic, and solid geometry.
The ablations resolve the corresponding roles: staged formalization determines the
quality of the initial shared state, admission controls which proposed transitions
enter the tree, neural bridges enlarge the reachable action space, and source
allocation determines where expansion effort is spent.

\paragraph{Search allocation and sensitivity.}
The outer and inner controllers solve different allocation problems. Source-UCB
chooses how to generate the next candidate set at a newly expanded leaf; PUCT then
revisits actions that have already entered the tree. This separation is reflected in
the fixed-schedule ablation, which changes source allocation while leaving the action
interface and traversal rule intact. In the search-constant sweep (Appendix~\ref{app:sensitivity}),
varying $c_{\rm puct}$ or $c_{\rm meta}$ over $\{0.5,1.0,1.5,2.0,3.0\}$ with the
other held at $1.5$ places the shared setting $c_{\rm puct}=c_{\rm meta}=1.5$ at the
highest observed accuracy in all three domains; solid geometry is especially
responsive to the outer source-allocation constant. This pattern aligns with the
consistent gain from adaptive source selection in Table~\ref{tab:component-ablation}.

\paragraph{Executable semantics as the reusable substrate.}
The two neural backends change proposal quality and finite-budget solve rate while
leaving the geometry state, registered theorem semantics, exact equation execution,
and provenance contract unchanged. The Claude and Qwen3-VL-8B-Thinking results
therefore exercise the same symbolic substrate with different proposal models. The
IMO-AG-30 experiment extends the same design to olympiad construction problems using
a compact executable library. Taken together, the experiments position the typed
state and transition contract as the stable interface through which model priors,
formal deduction, and inference-time search can be composed.

\section{Conclusion}
\label{sec:conclusion}

\OmniGeo formulates geometry reasoning as adaptive search over a shared typed
state. Neural guidance supplies multimodal formalization, structural proposals,
and action priors; symbolic execution supplies registered theorem application,
exact propagation, and provenance-aware state updates. The source-UCB/action-PUCT
controller repeatedly reallocates computation as this state evolves. Across
FormalGeo7K, Conic10K, and SolidFGeo, the resulting system reaches a $90.8\%$
macro average with Claude Sonnet 4.6 and $86.8\%$ with Qwen3-VL-8B-Thinking.
Component ablations localize gains to staged formalization, neural bridge
actions, typed admission, and adaptive source allocation, while the IMO-AG-30
result demonstrates the same interface in an olympiad setting with a compact
executable geometry library. Together, these results establish a unified
state-level contract for composing learned proposals with exact symbolic
execution across heterogeneous geometry domains.

\clearpage
\appendix
\section{Detailed \GeoLib and \GeoDL Specification}
\label{app:language}

This appendix specifies the shared-state interface introduced in the main
text. We describe its declarations and representative rules, followed by
condition normalization and the provenance records linking successive
transitions.

\subsection{Predicate and Attribution Declarations}

Every symbol has a declaration
\begin{equation}
\begin{aligned}
\operatorname{PredicateDef}(p)=\{&\mathrm{name},\mathrm{args},
\mathrm{existence},\mathrm{validity},\\
&\mathrm{multi},\mathrm{expansion},\mathrm{symbol}\}.
\end{aligned}
\end{equation}
The fields specify the predicate name, typed arguments, existence and validity
checks, multiplicity, expansion rule, and associated algebraic symbol.
Relation predicates return truth values; attribution predicates create exact
algebraic symbols. Table~\ref{tab:predicate-hierarchy} lists representative
surface forms, organized into classes that share one formal signature.

\begin{table*}[ht]
\centering
\caption{Representative declarations in the unified signature.}
\label{tab:predicate-hierarchy}
\setlength{\tabcolsep}{5pt}
\begin{tabular}{p{0.16\linewidth}p{0.27\linewidth}p{0.49\linewidth}}
\toprule
Class & Role & Examples\\
\midrule
Construction & Conservative introduction of entities &
$\gpred{Shape}(ABC)$, $\gpred{Collinear}(ABC)$,
$\gpred{Cocircular}(O,AB)$, $\gpred{Coplanar}(U,ABC)$,
$\gpred{Cospherical}(O,ABC)$\\
Entity & Typed geometric objects &
$\gpred{Triangle}(ABC)$, $\gpred{Circle}(O)$, $\gpred{Plane}(U)$,
$\gpred{Sphere}(O)$, $\gpred{Cube}(ABCDEFGH)$,
$\gpred{Cylinder}(P,Q)$, $\gpred{Conic}(\Gamma)$\\
Relation & Qualitative facts &
$\gpred{ParallelBetweenLine}(AB,CD)$,
\par\resizebox{0.82\linewidth}{!}{$\gpred{PerpendicularBetweenLineAndPlane}(AB,U)$},
$\gpred{OnCurve}(A,\Gamma)$, $\gpred{Tangent}(l,\Gamma)$\\
Attribution & Exact geometry--algebra bridge &
$\gpred{LengthOfLine}(AB)$, $\gpred{MeasureOfAngle}(ABC)$,
$\gpred{Coordinate}(A,x_A,y_A)$, $\gpred{RadiusOfSphere}(O)$,
$\gpred{VolumeOfCylinder}(P,Q)$\\
\bottomrule
\end{tabular}
\end{table*}

For example, a neural midpoint construction takes arguments $(M,A,B)$.
Admission checks their arity, creates missing point symbols and the
corresponding line and angle extensions, and adds
\begin{equation*}
\gpred{Collinear}(AMB),\qquad
\gpred{IsMidpointOfLine}(M,AB).
\end{equation*}
These relations are recorded as bridge conditions under the
\texttt{llm\_construction} tag. They can be grounded by input conditions or
subsequent registered theorem applications. The declarative midpoint
expansion is
\begin{equation*}
\begin{aligned}
\gpred{Equal}\bigl(&\gpred{LengthOfLine}(AM),\\[-1mm]
&\gpred{LengthOfLine}(MB)\bigr).
\end{aligned}
\end{equation*}
The registered midpoint theorem introduces this metric equality during exact
execution.

\subsection{Representative Rule Schemas}

A representative plane rule is
\begin{align}
&\gpred{Equal}\bigl(\gpred{LengthOfLine}(AM),\nonumber\\[-1mm]
&\hspace{15mm}\gpred{LengthOfLine}(MB)\bigr)
\wedge\gpred{Collinear}(AMB)\nonumber\\
&\hspace{20mm}\Rightarrow\gpred{IsMidpointOfLine}(M,AB).
\end{align}
A registered solid-incidence rule is
\begin{align}
&\gpred{PerpendicularBetweenLine}(AD,BD)\wedge\nonumber\\[-1mm]
&\gpred{PerpendicularBetweenLine}(AD,CD)\wedge\nonumber\\[-1mm]
&\gpred{Coplanar}(U,BCD)
\wedge\neg\gpred{Coplanar}(U,AD)\nonumber\\
&\Rightarrow\nonumber\\[-1mm]
&\scalebox{0.82}{$\gpred{PerpendicularBetweenLineAndPlane}(AD,U)$}.
\end{align}
The last premise is the incidence guard that excludes $AD$ from plane $U$.
A coordinate rule may introduce
\begin{equation}
\gpred{Distance}(A,B)^2=(x_A-x_B)^2+(y_A-y_B)^2
\end{equation}
from typed coordinate declarations.  A conic tangency rule can reduce a
line--curve intersection to a discriminant constraint, after which the exact
equation engine propagates the recorded relation.  A metric solid rule derives
\begin{align}
&\gpred{Cylinder}(P,Q)\Rightarrow
\gpred{Equal}\bigl(\gpred{VolumeOfCylinder}(P,Q),\nonumber\\[-1mm]
&\hspace{12mm}\gpred{Mul}(\gpred{AreaOfCircle}(P),\nonumber\\[-1mm]
&\hspace{24mm}\gpred{HeightOfCylinder}(P,Q))\bigr)
\end{align}
under the corresponding entity and height declarations.

\subsection{Condition Normalization and Duplicate Control}

Predicate declarations define permitted symmetric or multi-argument
extensions, and the condition pool groups items by predicate.  Exact
expressions are normalized by the symbolic algebra backend, condition
insertion suppresses duplicates, and each action has a stable key derived
from its kind and payload. These checks reduce duplicate children locally.
Each MCTS node retains a separate copy of its condition object; nodes with
equivalent fact sets remain distinct because the search does not perform
global state hashing or node merging.

The declaration-level canonical forms include
\begin{equation}
\begin{gathered}
\gpred{Line}(AB)\equiv\gpred{Line}(BA),\\
\gpred{Equal}(x,y)\equiv\gpred{Equal}(y,x),\\
\gpred{ParallelBetweenPlane}(U,V)\equiv
\gpred{ParallelBetweenPlane}(V,U).
\end{gathered}
\end{equation}

\subsection{Provenance and Assumption Ledger}

The autoformalization record preserves separate construction, text, and image
\GeoDL arrays.  At runtime, root conditions use premise marker $-1$; theorem
conditions store premise IDs and a $(\mathsf{name},\mathsf{branch},
\mathsf{parameters})$ tuple; neural equations and constructions store
\texttt{llm} and \texttt{llm\_construction}; and internal extensions or
equation solves use reserved bookkeeping names.  The saved \GeoDL arrays retain
the input source channel, while runtime records retain transition provenance.
Together they support the goal-dependency slice $U_g$ used in
Theorem~\ref{thm:soundness}.

\section{Proofs of the Formal Guarantees}
\label{app:proofs}

The guarantees distinguish the logical status of a derivation from the
search process used to discover it. The statements below retain any open
neural dependencies explicitly and specify the conditions under which a
returned derivation is grounded.

\subsection{Formal Statements}
\label{sec:guarantees}

We formalize the logical and operational guarantees of the typed transition
and search protocol. Registered execution and exact propagation support
conditional soundness; state isolation and finite budgets constrain
execution; and source allocation and PUCT govern exploration. Let
$\mathfrak M$ denote the intended semantics of the compiled signature.
The following statements specify these guarantees, with expanded proofs and
source-accounting details in the subsequent subsections.

\begin{assumption}[Certified transition contract]
\label{ass:transition}
For every enabled registered theorem branch $r$, successful premise matching
and conclusion execution are sound in $\mathfrak M$.

Every deterministic
extension admitted to the certified transition registry preserves the meaning
of its source facts.

Exact propagation emits a value only when it is entailed
by the recorded equations together with their typed domain and sign
constraints; in particular, selecting one merely consistent root of a
non-unique system is not a certified propagation step.
\end{assumption}

\begin{definition}[Provenance-complete goal record]
A returned goal record is \emph{provenance-complete} when the ledger contains
its premise identifiers and every transition in their transitive ancestry.
Let $U$ be the set of recorded open neural facts and $U_g\subseteq U$ those
in this dependency slice.
\end{definition}

\paragraph{Returned record classes.}
The system classifies returned records as
\texttt{grounded}, \texttt{open\_obligation}, or \texttt{procedural}.  The
following result applies to provenance-complete records in the first two
classes; procedural outputs are reported separately and are not treated as
proof certificates.

\begin{theorem}[Soundness of returned derivations]
\label{thm:soundness}
If the \OmniGeo procedure returns a provenance-complete goal record $g$ and
Assumption~\ref{ass:transition} holds for every transition in its dependency slice, then
\begin{equation}
\mathfrak M\models
\left(\bigwedge_{f\in\Gamma_0}f\wedge\bigwedge_{u\in U_g}u\right)
\Rightarrow g.
\label{eq:obligation-soundness}
\end{equation}
When $U_g=\varnothing$, the returned goal follows from the compiled input
assumptions $\Gamma_0$ through certified registered, extension, and exact
transitions.

This is soundness relative to the compiled instance; fidelity of
$\Gamma_0$ to the raw text and diagram is a separate compilation property.
\end{theorem}

\begin{proof}[Proof sketch]
Induct over the premise records reachable from $g$.  Leaves are compiled
assumptions or explicitly tagged members of $U_g$.  Registered theorem steps
are sound by Assumption~\ref{ass:transition}, and exact propagation is sound relative to its
recorded equation parents.  Canonical extensions preserve equivalence by
the canonicalization contract. Composition yields
Eq.~\eqref{eq:obligation-soundness}, using exactly the compiled assumptions
and bridge conditions in the goal ancestry.
\end{proof}

\begin{corollary}[Separation of proposal and verification]
For any deterministic or stochastic proposal process, every
provenance-complete returned trace with $U_g=\varnothing$ remains valid
relative to $\Gamma_0$.  Proposal quality affects which valid traces are
reached within budget, but not the validity of a trace accepted by the
certified transition contract.
\end{corollary}

\begin{proposition}[Parent and sibling isolation]
\label{prop:isolation}
Let $s'$ be a child materialized from $s$ by action $a$.  Applying $a$ cannot
change $\Gamma_s$, its goal flags, or the semantic state of any sibling of $s'$.  A
failed or subsequently low-valued neural assertion is therefore confined to
the subtree rooted at $s'$.
\end{proposition}

\begin{proof}[Proof sketch]
The implementation first deep-copies the parent's condition object, installs
the copy, applies the action, snapshots the child, and reinstalls the saved
parent condition and goal tuple.  No MCTS node shares a mutable condition
object with its siblings; diagnostic timing data and equation caches are not
members of the semantic state and cannot add conditions without the installed
node record.
\end{proof}

\begin{proposition}[Finite-budget termination and checking cost]
\label{prop:termination}
Let $I$ be the number of completed search-loop iterations and $B\le I$ the
number of source-expansion events before the reported wall-clock deadline,
$D$ the depth bound, and $K_{\rm sym},K_{\rm nn}$ the largest raw candidate
lists returned by the two sources in one attempt.

Because an expansion has
at most one primary and one fallback source, the run performs at most $ID$
tree-edge selections and
$B(K_{\rm sym}+K_{\rm nn})+K_{\rm rescue}$ admission attempts, where
$K_{\rm rescue}$ is the optional one-time longer symbolic root enumeration.
It uses at most $L$ neural rounds and terminates under the worker hard deadline.

For a fixed theorem substitution with $m$ premises, record lookup and
conclusion admission cost
$O(mC_{\rm lookup}+C_{\rm guard}+C_{\rm exec})$. The cost of finding
substitutions through predicate-indexed matching is separate and is limited
operationally by a per-call timer.
\end{proposition}

\paragraph{Budget implementation.}
The runner uses a nominal 600-second search cap, a 30-second worker-watchdog
grace interval, and a parent-process kill deadline.  Symbolic matching and
exact algebra consume the active allowance; on timeout, the runner returns the
last completed status.

\begin{proposition}[Bounded neural use and fallback]
At most $L$ neural proposal rounds are attempted.  A failed semantic round
consumes one round, so repeated failures cannot create an infinite retry path.

Whenever a chosen source returns no admissible child and the other source is
available, the implementation records zero for the primary and tries the
other source.  The fallback receives $V(s)$ if it produces an admissible child;
if it is also empty, it receives zero.

After $L$ neural rounds, all later expansions
use the symbolic source.
\end{proposition}

\begin{proposition}[Source-credit boundedness]
\label{prop:source-credit}
For an attempted source $E$, $R_E(s)$ in
Eq.~\eqref{eq:source} depends only on $V(s)$ and whether $E$ produces a
nonempty post-filter expansion.  It lies in $[0,1]$, is zero when no child is
materialized, and is independent of child values, child priors, and candidates
proposed only by another source.
\end{proposition}

\begin{proposition}[Heuristic subordination]
For a fixed visited state, action filtering, theorem execution, and
$\GoalCheck$ are independent of the bounded value function used for search
ordering. Replacing the bounded progress function $V$ by a bounded $V'$ may
change source allocation and the visitation policy while preserving terminal
verification and the transition ledger.
\end{proposition}

\paragraph{Interpretation.}
The bounded progress heuristic orders partial states; exact terminal recognition is
performed by $\GoalCheck$.  Two states may share the same progress features
while differing in whether a complete proof path remains reachable.

\begin{theorem}[Local exploration of persistently live positive-prior children]
\label{thm:coverage}
Fix an expanded node $s$ with a finite, unchanged set of children that remain
live and bounded returns $Q(s,a)\in[0,1]$.  If $s$ is selected infinitely
often, then under Eq.~\eqref{eq:puct} every persistently live child with
$P(s,a)>0$ is selected infinitely often.
\end{theorem}

\begin{proof}[Proof sketch]
Within a finite child set, the exploration term of a positive-prior child
with bounded visits eventually dominates those of all infinitely visited
children, accounting for bounded $Q$ values. PUCT must then revisit the set
of children with bounded counts, contradicting the eventual constancy of
those counts.
\end{proof}

\paragraph{Scope.}
The result concerns children already materialized at a
revisited node; source selection at new leaves is governed separately by
Eq.~\eqref{eq:source}.

\subsection{Soundness of Each Transition Class}

\begin{lemma}[Theorem transition]
If $\Gamma\vdash_{\GeoLib}(r,\sigma)\Downarrow(\Delta,\rho_a)$, then every
$\delta\in\Delta$ is a semantic consequence of $\Gamma$ in $\mathfrak M$.
\end{lemma}
\begin{proof}
The executor resolves the registered branch, matches concrete premise IDs,
binds the declared variables, and admits only its checked conclusions.
Assumption~\ref{ass:transition} states that
$P_r\wedge G_r\Rightarrow C_r$ is valid in $\mathfrak M$.  Therefore each
instantiated conclusion follows.
\end{proof}

\begin{lemma}[Exact propagation]
If the equation engine derives $h=0$ solely from recorded parent equations
$f_1=\cdots=f_m=0$, then under Assumption~\ref{ass:transition}
$\{f_i=0\}_{i=1}^m\models h=0$.
\end{lemma}
\begin{proof}
Assumption~\ref{ass:transition} specifies that exact equation propagation preserves semantic
consequence with respect to its recorded parent equations.  Applying this
contract to $f_1,\ldots,f_m$ yields the stated implication.
\end{proof}

\begin{lemma}[Neural bridge transition]
If $\Gamma\Vdash_{\rm filt}a\Downarrow\Delta$ for a neural construction or
asserted equation, its typed semantic contract is
$\Gamma\wedge\Delta\models\Delta$.
\end{lemma}
\begin{proof}
The parser, novelty, anti-injection, and consistency tests constrain the form
of $\Delta$.  Recording $\Delta$ in the bridge set supplies the assumption
used by the conditional implication.
\end{proof}

\begin{proof}[Proof of Theorem~\ref{thm:soundness}]
Take the transitive premise ancestry of
the returned goal.  Replace every theorem or exact-propagation node by its
sound implication and every open node by an assumption in $U_g$.  Composition
over this finite acyclic dependency relation yields
Eq.~\eqref{eq:obligation-soundness}.
\end{proof}

\subsection{State Isolation and Local PUCT Exploration}

\begin{proof}[Proof of Proposition~\ref{prop:isolation}]
Let $C_s$ denote the mutable condition
object installed for parent $s$.  Child materialization computes
$C'=\mathrm{deepcopy}(C_s)$, installs $C'$, applies the action, and stores $C'$
in the child state.  It then reinstalls the saved pointer $C_s$ and the saved
goal tuple.  Another child starts from a new deep copy of $C_s$, never from
$C'$.  Induction over child creation proves parent and sibling isolation.
\end{proof}

\Needspace{17\baselineskip}
\begin{proof}[Proof of Theorem~\ref{thm:coverage}]
Partition the fixed finite child set into
children with bounded visit counts and those visited infinitely often.
Suppose the first set contains a child $a$ with $P(s,a)>0$ and
$N(s,a)\leq M$. Its PUCT exploration term is at least
\begin{equation}
c_{\rm puct}P(s,a)\frac{\sqrt{N(s)}}{1+M},
\end{equation}
which grows proportionally to $\sqrt{N(s)}$. For each infinitely visited
child, the exploration term divided by $\sqrt{N(s)}$ tends to zero; the same
holds for every bounded $Q$ term. Since the child set is finite, PUCT must
eventually select from the set with bounded visit counts. Yet the counts in
that finite set must eventually be constant, a contradiction. Thus every
persistently live positive-prior child is visited infinitely often.
\end{proof}

\subsection{Finite Budgets and Source Accounting}

\paragraph{Finite budgets.}
Each neural proposal invocation increments the round counter after the
two-call attempt, whether parsing succeeds or fails.  Hence the guard
$\textit{rounds}<L$ can be true at most $L$ times.  Symbolic enumeration is
wrapped by a per-call alarm; the root receives at most one longer rescue
attempt when the first enumeration returns no child and sufficient engine time
remains.  Search depth is capped at $D$, and the runner enforces an outer hard
deadline independent of the MCTS loop.

\paragraph{Source accounting.}
For every attempted source, the implementation evaluates the current leaf,
computes Eq.~\eqref{eq:source} from $V(s)$ and expansion success, and
updates exactly one count/value pair from the expanded leaf through its
ancestors.  An empty primary and a successful fallback therefore create two
separate observations, $(E_0,0)$ and $(E_1,V(s))$.  Repeated symbolic
enumeration inside the one-time root rescue is folded into the same symbolic
observation.  A descendant with no local history reads information from the
closest ancestor with nonzero counts.  This produces the path-conditioned
online allocation rule in Eq.~\eqref{eq:source}; the local child-exploration
result is supplied separately by Theorem~\ref{thm:coverage}.

\section{Implementation and Evaluation Details}
\label{app:impl}

This section describes the neural interfaces that implement the transition
contract, then specifies the evaluation protocol and shared search settings.

\subsection{Neural Backend, Prompts, and Token Accounting}

We evaluate two primary neural-backend configurations: Claude Sonnet 4.6 and
Qwen3-VL-8B-Thinking.  Within each configuration, all neural calls use the same
off-the-shelf backend and temperature 0.  The Claude configuration uses the
\texttt{claude-sonnet-4-6} model through an OpenAI-compatible Chat Completions
endpoint with \texttt{openai==2.31.0} and a 200k-token context window.  Both
configurations keep the selected checkpoint frozen throughout compilation
and search, without geometry-specific training.

Table~\ref{tab:neural-interface} gives the complete call interface, and
Appendix~\ref{app:prompts} reproduces the fixed prompt contracts.  Across
compiler calls, the templates require the declared JSON schema and names from
the supplied \GeoLib whitelist; the in-search extractor additionally forbids
a direct goal-answer equation.  Stage~1 receives the
statement, optional raster diagram, allowed construction symbols, and
representative typed examples, and returns a JSON list of construction declarations.
Stage~2 additionally receives Stage~1 and returns
\texttt{text\_geodl}/\texttt{image\_geodl}.  Stage~3 receives the statement and
compiled conditions and returns a typed \texttt{goal\_geodl} plus the auxiliary
model-generated \texttt{problem\_answer} string $\tilde y_\theta$ described in
Appendix~\ref{app:eval}; $y^\star$ is not included.  The critic receives the
same raw statement and candidate \GeoDL object and
returns \texttt{\{"ok": bool, "missing": string\}}.  During search, the
planner receives the current conditions, known values, equations, goal, and
backward subgoal hints.  The extractor receives that plan plus the engine
symbol table and returns at most three typed actions with priors.

\begin{table*}[t]
\centering
\caption{Neural call interface used in every benchmark.  ``Cap'' is the
maximum completion length per call.}
\label{tab:neural-interface}
\setlength{\tabcolsep}{4pt}
\resizebox{\textwidth}{!}{%
\begin{tabular}{p{0.14\linewidth}p{0.31\linewidth}p{0.34\linewidth}cc}
\toprule
Role & Inputs & Required output & Cap & Retry/round limit\\
\midrule
Construction compiler & statement, optional image, typed whitelist, examples & JSON list of construction \GeoDL strings & 4096 & 1\\
Condition compiler & statement, image, construction \GeoDL, typed whitelist, examples & JSON object with text/image \GeoDL lists & 4096 & 1\\
Goal compiler & statement and compiled \GeoDL & typed \texttt{goal\_geodl} and model-generated $\tilde y_\theta$ & 4096 & 1\\
Compiler critic & statement and candidate \GeoDL & JSON \texttt{ok}/\texttt{missing} record & 2048 & 2 rounds\\
Search planner & current formal state, equations, goal, subgoal hints & geometry plan naming concrete next steps & 8192 & 3 rounds/problem\\
Action extractor & plan, engine symbols, equations, action schema & $\leq3$ typed actions with scores in $[0,1]$ & 4096 & paired with planner\\
\bottomrule
\end{tabular}}
\end{table*}

\subsection{Reproducible Prompt Templates}
\label{app:prompts}

The following boxes reproduce the prompt instructions used by the released
implementation, with only line wrapping and paper notation normalized.
Tokens in angle brackets are populated at runtime from
the active domain, \GeoLib whitelist, few-shot file, candidate formalization,
or current search state. The boxes therefore specify the instruction and
output contracts while leaving problem-dependent symbol tables as inputs. Retry
feedback is appended only after a failed parse or engine attempt.  No
benchmark reference answer is inserted into any prompt.  The implementation
retains legacy JSON keys ending in \texttt{\_cdl}; these keys carry the
corresponding \GeoDL channels and do not denote a second formal language.

\begin{promptbox}{Prompt A: Construction compiler}
You parse a natural-language <DOMAIN> geometry problem into the
construction part of a GeoDL specification.

OUTPUT FORMAT: a single JSON array of strings. Reply with ONLY the
JSON array; no preamble, explanation, markdown, or trailing prose.
The first character must be [ and the last must be ]. Each element
names one entity or one multi-entity construction.

Allowed construction predicates: <CONSTRUCTION_WHITELIST>
Allowed basic entities: <BASIC_ENTITY_WHITELIST>
Allowed compound entities: <COMPOUND_ENTITY_WHITELIST>

CONSTRAINTS:
1. Use only registered predicates and entities.
2. Balance every pair of parentheses before emitting a GeoDL string.
3. If a figure is attached, use it only for structural facts that the
   text does not state (e.g., adjacency, collinearity, incidence, or
   cyclic vertex order).
4. Text labels are canonical. Never introduce a point label found only
   in the figure, and never invent a shape absent from the statement.
5. Preserve the granularity of named objects; do not decompose a solid
   or polygon into unsupported sub-objects.
6. Apply <DOMAIN_CONSTRUCTION_GUIDE>.

Examples:
<CONSTRUCTION_EXAMPLES>

NL: <NATURAL_LANGUAGE_STATEMENT>
[Optional figure attached]
[Optional parser feedback: <PREVIOUS_FAILURE>]
CONSTRUCTION:
\end{promptbox}

\begin{promptbox}{Prompt B: Condition compiler}
You parse the relations, equations, and inequalities of a <DOMAIN>
GeoDL specification, given the natural-language statement and the
already extracted construction.

OUTPUT FORMAT:
- With a figure: exactly
  {"text_cdl":["..."], "image_cdl":["..."]}
- Without a figure: exactly {"text_cdl":["..."]}
Return ONLY the JSON object; no prose or markdown.

Allowed relations: <RELATION_WHITELIST>
Allowed attributions: <ATTRIBUTION_WHITELIST>
Allowed operations: Add, Sub, Mul, Div, Pow, Mod, Sqrt, Sin, Cos,
                    Tan, Abs, Neg, ...
Equation form: Equal(LHS,RHS)
Inequality forms: Greater, Less, GreaterEqual, LessEqual, Negation

CONSTRAINTS:
1. text_cdl contains every relation, equation, and inequality stated
   or unambiguously implied by the text. The text is authoritative.
2. With a figure, image_cdl begins with an identical, ordered copy of
   text_cdl and then appends only figure-specific structural facts.
3. Use the same labels in every channel. Drop figure-only labels and
   any image fact that contradicts the text.
4. Do not invent numeric values or fresh placeholder symbols. Values
   in examples are style demonstrations, not sources for this problem.
5. Encode expressions with registered operators (e.g., Sqrt, Sin),
   and balance every pair of parentheses.
6. Translate implicit geometry into registered relations: incidence,
   collinearity at intersections, midpoint, perpendicular foot,
   perpendicular bisector, tangency, and cocircularity when stated.
7. Apply <DOMAIN_CONDITION_GUIDE>.

Examples:
<CONDITION_EXAMPLES>

NL: <NATURAL_LANGUAGE_STATEMENT>
CONSTRUCTION: <CONSTRUCTION_CDL>
[Optional figure and parser feedback: <PREVIOUS_FAILURE>]
JSON:
\end{promptbox}

\begin{promptbox}{Prompt C: Goal compiler}
You extract the goal of a <DOMAIN> geometry problem and a
model-generated answer string used only for format compatibility.

OUTPUT FORMAT: exactly
{"goal_cdl":"<ONE_STRING>", "problem_answer":"<ONE_STRING>"}
Return ONLY this JSON object; no prose or markdown.

goal_cdl starts with one of:
Value, FindRange, FindMin, FindMax, FindExpression, FindLocus,
FindCoordinate, Equal, or Relation.

CONSTRAINTS:
1. For a prove/Relation goal, set problem_answer to "0".
2. FindExpression(X) takes one entity name X. To request the value of
   an arithmetic expression, use Value(<EXPRESSION>).
3. Use only <ATTRIBUTION_WHITELIST> and <OPERATION_WHITELIST>.
4. Balance every pair of parentheses in goal_cdl.
5. problem_answer is inferred by the model; the benchmark reference
   answer is never supplied to this prompt or to the solver.

Examples:
<GOAL_EXAMPLES>

NL: <NATURAL_LANGUAGE_STATEMENT>
CONSTRUCTION: <CONSTRUCTION_CDL>
TEXT: <TEXT_CDL>
IMAGE: <IMAGE_CDL>
[Optional parser feedback: <PREVIOUS_FAILURE>]
JSON:
\end{promptbox}

\begin{promptbox}{Prompt D: Compiler critic and repair diagnosis}
You are a critic that checks whether a GeoDL specification correctly
and completely captures a natural-language <DOMAIN> problem.

Audit ONLY construction_cdl, text_cdl, and image_cdl. Do not propose
changes to goal_cdl or problem_answer; the goal is fixed separately.
Enumerate every entity, numeric or symbolic value, relation, equation,
inequality, and implicit geometric convention in the statement, and
compare them with the candidate GeoDL.

OUTPUT FORMAT: exactly
{"ok":true|false, "missing":"<EMPTY_IF_OK; OTHERWISE_DIAGNOSIS>"}
Return ONLY this JSON object; no prose or markdown.

If <ENGINE_FEEDBACK> is present, diagnose the failure in concrete
GeoDL terms. Distinguish an unregistered operator, an undeclared
entity, a missing incidence/construction, a name collision, and an
accepted but underconstrained goal. In missing, tell the next compiler
pass which line(s) to add or replace, verbatim where possible.

Apply <DOMAIN_CRITIC_GUIDE>.
Examples: <CRITIC_EXAMPLES>

NL: <NATURAL_LANGUAGE_STATEMENT>
CDL: <CANDIDATE_CDL_OBJECT>
[ENGINE FEEDBACK: <SOLVED, DERIVED_ANSWER, ERROR>]
[PREVIOUS AUDIT AND PATCH FAILURE: <RETRY_FEEDBACK>]
JSON:
\end{promptbox}

\begin{promptbox}{Prompt E: State-conditioned search planner}
SYSTEM:
You are an expert in <DOMAIN> geometry.
- Plane: triangles, polygons, circles, parallel/perpendicular lines,
  angle chasing, similarity, congruence, and classical centers.
- Analytic: conics, circles, lines, coordinates, slopes, distances,
  and conic-parameter identities.
- Solid: cones, cylinders, prisms, pyramids, cubes, spheres,
  cross-sections, and line/plane relations.

You are given a problem and the symbolic engine's current facts and
equations. Reason step by step toward the goal. Name the specific
registered theorem or geometric identity for each proposed step and
state the concrete entities, coordinates, or equations involved.
Produce reasoning only; this call does not mutate the engine state.

USER STATE:
Problem: <NATURAL_LANGUAGE_STATEMENT>

Conditions:
  <CONSTRUCTION_AND_TEXT_CDL>

Known:
  <SYMBOL (ATTRIBUTION(ARGS)) = VALUE>

Equations:
  <ENGINE_EXPRESSION = 0>

Goal: <GOAL_CDL>
<OPTIONAL_BACKWARD_SUBGOAL_HINTS>
\end{promptbox}

\begin{promptbox}{Prompt F: Typed action extractor}
You are the INJECTION step of a neural-symbolic geometry solver. A
planner has already reasoned about the problem. Select the SINGLE most
important next fact it established and encode that fact as one engine
action. Emit multiple actions only when several relations must hold
together for that one step; emit at most three actions.

Copy every engine symbol exactly from the supplied state tables; never
invent or guess a symbol. Copy theorem names verbatim from the planner.

OUTPUT: ONLY a JSON array. Every object is one of:
{"kind":"theorem", "name":"<NAME>", "branch":<INT_OR_NULL>,
 "params":["A","B","C"], "prior":<0_TO_1>, "why":"<TEXT>"}
{"kind":"assert_equation", "expr":"<EXPR_IN_=0_FORM>",
 "prior":<0_TO_1>, "why":"<TEXT>"}
{"kind":"assert_equations", "exprs":["<E1>","<E2>"],
 "prior":<0_TO_1>, "why":"<TEXT>"}
{"kind":"construction", "construction":"Midpoint|Line|Collinear|Point",
 "args":["M","A","B"], "prior":<0_TO_1>, "why":"<TEXT>"}

Use assert_equations when relations must be solved jointly. Prefer
linear relations when available. RED LINE: never inject the answer.
An asserted equation whose only free symbols are goal symbols is
rejected; every proposed relation must connect the goal to other
engine symbols. This test is applied to the union of symbols in a
joint assertion.

Goal: <GOAL_CDL>
Engine state -- known symbols: <KNOWN_SYMBOL_TABLE>
Engine state -- unknown symbols: <UNKNOWN_SYMBOL_TABLE>
Existing equations: <ENGINE_EQUATIONS>
Available constructions: Midpoint, Line, Collinear, Point
Available operations: Add, Sub, Mul, Div, Pow, Sqrt, Sin, Cos,
                      Tan, Abs, Neg

Reasoning from the planner:
<PLANNER_REASONING>

Output ONLY a JSON array of actions.
\end{promptbox}

The default limits allow one additional retry per compiler stage, one additional
end-to-end compiler attempt, and two critic rounds per attempt.  Each response
contributes prompt/completion/total token counts and a call count to the
per-problem record; the searcher separately accumulates neural wall time.  The
search permits at most three planner/extractor pairs, corresponding to at most
six semantic in-search calls and 36,864 completion tokens;
transport retries are logged separately and do not create additional search
rounds.

\subsection{Input and Evaluation Protocol}
\label{app:eval}

Compiler and search inputs are limited to the benchmark statement and, for
multimodal problems, the official raster diagram.  The held-out answer
$y^\star$ is accessed by the scorer only after inference.  FormalGeo7K and
SolidFGeo use their official statement--diagram pairs, whereas Conic10K uses
its natural-language statement because its supplied schema contains no raster
diagram.  The separate IMO-AG-30 study uses the official diagram input.
General-purpose baselines receive the same modality as \OmniGeo on each split.
The compiler-generated $\tilde y_\theta$ is stored separately from the
benchmark label.

Each result record stores the resolved image path or manifest identifier.  For
the plane and solid multimodal protocols, the runner fails closed if the
official image cannot be resolved.  Text conditions can populate the image
channel only in an explicitly labeled text-only ablation, preventing silent
modality substitution in the main evaluation.

\paragraph{General-purpose baseline inference budget.}
Each general-purpose language-model baseline uses one semantic model call per
problem, with retries reserved for failed calls. Once a valid response is
returned, inference terminates; a
failed call may be retried up to five times with the identical prompt and
decoding configuration, and these retries are not treated as additional
independent samples.  The maximum completion length is 4096 tokens,
temperature is 0, and each individual attempt has a 180-second timeout.  Thus
the baseline budget contains no self-consistency voting or iterative
solver-feedback loop beyond failed-call recovery.

\clearpage
\subsection{Search Hyperparameters}

\begin{table}[htbp]
\centering
\caption{Shared inference settings.}
\label{tab:hyperparameters}
\resizebox{\columnwidth}{!}{%
\begin{tabular}{lc}
\toprule
Parameter & Value\\
\midrule
Nominal search cap & 600 s/problem\\
Worker watchdog grace & 30 s\\
$c_{\rm puct}$ & 1.5\\
$c_{\rm meta}$ & 1.5\\
Maximum tree depth & 15\\
Neural expansion rounds & 3/problem\\
Symbolic enumeration cap & 120 s/call\\
Proposal temperature & 0\\
Additional per-step retries & 1\\
Additional pipeline attempts & 1\\
Critic rounds & 2/attempt\\
Symbolic child prior & uniform\\
Post-filter prior & re-normalized over admitted children\\
Outer source return & $V(s)$ if expansion is nonempty; $0$ otherwise\\
Value weights & $(5,2,0.6,1,0.6,-0.4)$; bias $-0.8$\\
\bottomrule
\end{tabular}
}
\end{table}

The progress weights are set manually and held fixed across benchmarks. The
first five coefficients reward solved status, concrete goal bindings,
goal-related equations, matched backward subgoals, and equation growth.  The
sixth is a depth penalty.  Terminal states are set by the engine goal checker;
the value function orders nonterminal search states and defines only the
bounded progress signal used in source credit. Each output record stores the
source statistics and search constants for that problem, allowing aggregate
statistics to be recomputed from the saved records.

\subsection{Search-Control Sensitivity}
\label{app:sensitivity}

We vary one exploration constant while holding the other at $1.5$.  The sweep
uses $\{0.5,1.0,1.5,2.0,3.0\}$ for both $c_{\rm puct}$ and $c_{\rm meta}$.
The shared setting $c_{\rm puct}=c_{\rm meta}=1.5$ gives the highest observed
accuracy in all three geometry domains.  Analytic geometry remains relatively
stable around this point, while solid geometry responds more strongly to the
outer source-allocation constant.

\begin{figure}[ht]
  \centering
  \includegraphics[width=0.72\linewidth]{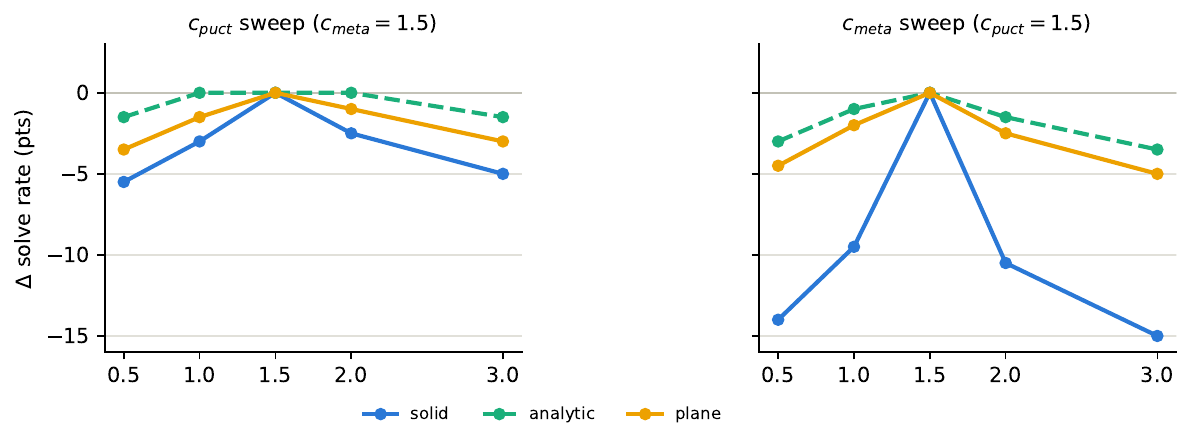}
  \caption{Sensitivity to the inner PUCT constant $c_{\rm puct}$ and the outer
  source-allocation constant $c_{\rm meta}$.  Each sweep varies one constant and
  holds the other at $1.5$.}
  \label{fig:search-sensitivity}
\end{figure}

\end{document}